\documentclass[11pt]{article}

\usepackage[margin=1in]{geometry}

\usepackage{amsmath}
\usepackage{amssymb}
\usepackage{amsthm}
\usepackage{mathtools}

\usepackage{booktabs}
\usepackage{array}

\usepackage{tikz-cd}
\usetikzlibrary{positioning, calc, patterns, decorations.pathreplacing}

\usepackage[colorlinks=true,linkcolor=blue,citecolor=blue,urlcolor=blue]{hyperref}

\usepackage{natbib}

\usepackage{listings}
\usepackage{xcolor}

\usepackage{enumitem}
\usepackage{microtype}

\newtheorem{theorem}{Theorem}[section]

\newtheorem{corollary}[theorem]{Corollary}
\newtheorem{lemma}[theorem]{Lemma}
\theoremstyle{definition}
\newtheorem{definition}[theorem]{Definition}
\theoremstyle{remark}

\definecolor{codebg}{rgb}{0.95,0.95,0.95}
\definecolor{codegreen}{rgb}{0,0.6,0}
\definecolor{codepurple}{rgb}{0.58,0,0.82}
\definecolor{codeblue}{rgb}{0.0,0.0,0.7}

\lstdefinestyle{coq}{
  backgroundcolor=\color{codebg},
  basicstyle=\ttfamily\small,
  breaklines=true,
  frame=single,
  framerule=0pt,
  xleftmargin=1em,
  xrightmargin=1em,
  aboveskip=0.5em,
  belowskip=0.5em,
  keywordstyle=\color{codeblue}\bfseries,
  commentstyle=\color{gray}\itshape,
  stringstyle=\color{codegreen},
  morekeywords={Definition,Theorem,Lemma,Corollary,Proof,Qed,
    Inductive,Record,Fixpoint,CoFixpoint,Section,End,Variable,
    Hypothesis,forall,exists,fun,match,with,end,if,then,else,
    let,in,return,Type,Prop,Set,Variant,Context,Notation},
  morecomment=[s]{(*}{*)},
  morestring=[b]",
  sensitive=true,
  showstringspaces=false,
}

\lstdefinestyle{ocaml}{
  backgroundcolor=\color{codebg},
  basicstyle=\ttfamily\small,
  breaklines=true,
  frame=single,
  framerule=0pt,
  xleftmargin=1em,
  xrightmargin=1em,
  aboveskip=0.5em,
  belowskip=0.5em,
  keywordstyle=\color{codeblue}\bfseries,
  commentstyle=\color{gray}\itshape,
  morekeywords={let,in,match,with,type,module,sig,struct,end,
    val,fun,function,if,then,else,open,rec,and,of,true,false},
  morecomment=[s]{(*}{*)},
  sensitive=true,
  showstringspaces=false,
}

\newcommand{\Mashin}{\textbf{Mashin}}
\newcommand{\CatM}{\mathbf{Mashin}}
\newcommand{\Gov}{\mathcal{G}}
\newcommand{\code}{\textsc{code}}
\newcommand{\reason}{\textsc{reason}}
\newcommand{\memory}{\textsc{memory}}
\newcommand{\mcall}{\textsc{call}}
\newcommand{\govsafe}{\texttt{gov\_safe}}
\newcommand{\withincaps}{\texttt{within\_caps}}
\newcommand{\itree}{\mathrm{itree}}
\newcommand{\DirectiveE}{\mathrm{DirectiveE}}
\newcommand{\IOE}{\mathrm{IOE}}
\newcommand{\GovIOE}{\mathrm{GovIOE}}
\newcommand{\GovE}{\mathrm{GovE}}
\newcommand{\Ret}{\mathrm{Ret}}
\newcommand{\Tau}{\mathrm{Tau}}
\newcommand{\Vis}{\mathrm{Vis}}
\newcommand{\bind}{\mathbin{>\!\!>\!\!=}}
\newcommand{\interp}{\mathrm{interp}}

\newcommand{\CapSet}{\mathrm{CapSet}}
\newcommand{\CapMorphism}{\mathrm{CapMorphism}}

\begin{document}

\title{Algebraic Semantics of Governed Execution:\\
Monoidal Categories, Effect Algebras, and Coterminous Boundaries}

\author{Alan L. McCann\\
\textit{Mashin, Inc.}\\
\texttt{research@mashin.live}}

\date{April 2026}

\maketitle

\begin{abstract}
We present an algebraic semantics for governed execution in which
governance is axiomatized, compositional, and coterminous with
expressibility. The framework, mechanized in 32~Rocq modules
(~12,000~lines, 454~theorems, 0~admitted), is built on interaction
trees and parameterized coinduction.

A three-axiom \emph{GovernanceAlgebra} record (safety,
transparency, properness) induces a symmetric monoidal category
with verified pentagon, triangle, and hexagon coherence, where
every tensor composition preserves governance. An algebraic
effect system constrains the handler algebra so that only
governance-preserving handlers can be constructed in the safe
fragment; programs in the empty capability set provably emit
only observability directives. Capability-indexed composition
bundles programs with machine-checked capability bounds, and
a dual guarantee theorem establishes that \texttt{within\_caps}
and \texttt{gov\_safe} hold simultaneously under all composition
operators.

The capstone result is the \emph{coterminous boundary}: within
our formal model, every program expressible via the four
primitive morphism constructors is governed under interpretation,
and every governed program is the image of such a program.
Turing completeness is preserved inside governance; unmediated
I/O is excluded from the governed fragment. Governance denial
is modeled as safe coinductive divergence.

The governance algebra is parametric: any system instantiating
the three axioms inherits all derived properties, including
convergence, compositional closure, and goal preservation.
Extracted OCaml runs as a NIF in the BEAM runtime,
with property-based testing (70,000+~random inputs, zero
disagreements) confirming behavioral equivalence between the
specification and the runtime interpreter.
\end{abstract}

\section{Introduction}
\label{sec:intro}

AI workflow systems compose language model calls with data
retrieval, code execution, and external API access. Governance
of these systems currently takes two forms. Behavioral
approaches (RLHF~\cite{ouyang2022training}, constitutional
AI~\cite{bai2022constitutional}, guardrails~\cite{rebedea2023nemo})
train or filter model outputs but do not address whether the
system's \emph{effects} (API calls, file writes, tool invocations)
pass through any governance boundary. Structural approaches enforce
invariants at the execution layer, but existing formalizations
either lack machine-checked proofs or do not address composition:
they verify individual programs rather than establishing that
governance is closed under the system's composition operators.

This paper develops an algebraic semantics for \emph{governed
execution}: a framework in which governance is axiomatized as a
three-property algebra, compositional under all program combinators,
and, within the formalized execution model, coterminous with
expressibility. The central thesis is:

\begin{quote}
\emph{Governance, when formalized as an algebraic structure over
interaction trees, induces compositional properties on programs,
effects, and capabilities such that the governed fragment coincides
with the expressible fragment in the modeled substrate.}
\end{quote}

\paragraph{Relationship to prior work in this line.}
Prior papers in this line establish structural governance at the
architectural level~\cite{mccann2026structural}, expressive
completeness via four primitives~\cite{mccann2026gcc}, and
machine-checked safety proofs for the specific
runtime~\cite{mccann2026mechanized}. The present paper abstracts
and unifies those system-specific results into an algebraic
semantics. The GovernanceAlgebra axiomatization lifts safety theorems
to parametric results that hold for \emph{any} system satisfying
three axioms. The monoidal, effect, and capability layers explain
why the system's programs compose under governance, not just why
individual programs are safe.
Two companion papers address practical enforcement:
\cite{mccann2026purity}~discharges the pure module constraint
assumed by the safety theorems, replacing convention-based
enforcement with WASM compilation and cryptographic purity
certificates; \cite{mccann2026provenance}~extends the governance
boundary to the supply chain with dual-signature distribution
provenance.

\paragraph{Three levels of ``governed.''}
We distinguish three levels that the paper addresses:
\begin{enumerate}[nosep]
  \item \textbf{Syntactic expressibility}: a program is constructible
    from the four primitive morphisms (\code, \reason, \memory, \mcall)
    via composition in category $\CatM$.
  \item \textbf{Semantic governance under interpretation}: interpreting
    a program through the governance operator $\Gov$ yields a tree
    satisfying the coinductive \govsafe{} predicate.
  \item \textbf{Trace-level evidence}: the execution trace extracted
    from a governed program is well-governed (every I/O event preceded
    by a governance check) and maps to a tamper-evident ledger.
\end{enumerate}
The coterminous boundary result connects levels 1 and 2:
every syntactically expressible program is semantically governed
under interpretation. Level 3 is a derived consequence
(Section~\ref{sec:traces}).

\paragraph{Contributions.}
The paper makes four contributions, each mechanized in Rocq:

\begin{enumerate}
  \item \textbf{GovernanceAlgebra and monoidal structure}
    (Sections~\ref{sec:algebra}--\ref{sec:monoidal}).
    A three-axiom parametric record from which convergence,
    compositional closure, and goal preservation are derived.
    Category~$\CatM$ admits a symmetric monoidal structure with
    verified pentagon, triangle, and hexagon coherence. All
    tensor compositions preserve governance.

  \item \textbf{Governed algebraic effects}
    (Section~\ref{sec:effects}).
    An algebraic effect system in which the handler algebra is
    constrained: only handlers carrying a machine-checked safety
    certificate (GovernedHandler) can be constructed in the governed
    fragment. The No Ambient Effects theorem proves that programs
    in the empty capability set emit only observability directives.

  \item \textbf{Capability-indexed composition with dual guarantee}
    (Section~\ref{sec:capability}).
    CapMorphisms bundle programs with capability bounds verified
    by a trust lattice. Composition operators (sequential, tensor,
    branch) preserve bounds. The dual guarantee establishes that
    \withincaps{} and \govsafe{} hold simultaneously.

  \item \textbf{Coterminous boundary}
    (Section~\ref{sec:coterminous}).
    Within the formalized execution model, every expressible program
    is governed, and every governed program is the image of an
    expressible program. Turing completeness is preserved inside
    governance. Unmediated I/O is excluded from the governed fragment.
\end{enumerate}

\noindent
Section~\ref{sec:traces} derives trace semantics and ledger
connection as consequences of the algebraic framework.
Section~\ref{sec:mechanization} discusses
mechanization and extraction to OCaml.
Section~\ref{sec:related} positions the work relative to prior
results on algebraic effects, monoidal categories, and
verified systems.

\paragraph{What is novel.}
The individual ingredients (monoidal categories, algebraic effects,
coinductive predicates, capability systems) are well-established.
To our knowledge, this is the first mechanized framework combining
governance-constrained interpretation, capability-bounded algebraic
effects, and a coterminous expressibility/governance theorem over an
interaction-tree-based execution model. The GovernanceAlgebra
axiomatization is parametric: the derived theorems hold for any
system satisfying three axioms, not only the concrete system we
instantiate.

\paragraph{The Interaction Trees framework.}
All formalizations use the Interaction Trees library~\cite{xia2020itrees},
which represents programs as coinductive trees with three node types:
pure values ($\Ret$), silent steps ($\Tau$), and visible events ($\Vis$).
The governance pipeline transforms trees of directive events ($\DirectiveE$)
into trees of governed events ($\GovIOE$). Coinductive properties are
proved using the paco library~\cite{hur2013paco} for parameterized
coinduction.

\section{Background}
\label{sec:background}

\subsection{Interaction Trees}

An interaction tree $\itree\ E\ R$ is a potentially infinite tree
coinductively defined by three constructors:
\begin{align*}
  \Ret(r) &\quad \text{pure value of type } R \\
  \Tau(t) &\quad \text{silent computation step} \\
  \Vis(e, k) &\quad \text{event } e : E\;X, \text{ continuation } k : X \to \itree\ E\ R
\end{align*}

Interaction trees form a monad: $\mathrm{ret}$ injects values and
$\bind$ sequences computations. Interpretation via handlers
$h : \forall X.\; E\;X \to \itree\;F\;X$
transforms $\itree\;E\;R$ into $\itree\;F\;R$.
Equivalence is \emph{eutt} (equivalence up to taus): two trees are
equivalent if they produce the same visible events and return values,
ignoring silent steps.

\subsection{The Directive Type}

The system's effect signature is $\DirectiveE$, an inductive type
with 14 constructors covering all external capabilities:

\begin{lstlisting}[style=coq]
Inductive DirectiveE : Type -> Type :=
  | LLMCall : LLMCallParams -> DirectiveE LLMResponse
  | HTTPRequest : HTTPRequestParams -> DirectiveE HTTPResponse
  | FileOp : FileOpParams -> DirectiveE FileResult
  | CallMachine : CallMachineParams -> DirectiveE CallMachineResult
  | MemoryOp : MemoryOpParams -> DirectiveE MemoryResult
  | DBOp : DBOpParams -> DirectiveE DBResult
  | ExecOp : ExecOpParams -> DirectiveE ExecResult
  | RecordStep : RecordStepParams -> DirectiveE unit
  | Broadcast : BroadcastParams -> DirectiveE unit
  | EmitEvent : EmitEventParams -> DirectiveE unit
  | GraphQLRequest : GraphQLRequestParams -> DirectiveE HTTPResponse
  | WebSocketOp : WebSocketOpParams -> DirectiveE WebSocketResult
  | MCPCall : MCPCallParams -> DirectiveE CallMachineResult
  | Observability : string -> DirectiveE unit.
\end{lstlisting}

The governed event type layers governance checks over I/O:
$\GovIOE = \GovE + \IOE$, where $\GovE$ carries governance stage checks
and $\IOE$ carries actual I/O events.

\subsection{The Governance Operator}

The operator $\Gov : \mathrm{base\_handler} \to \mathrm{governed\_handler}$
wraps a base handler with governance checks. Before each I/O operation,
$\Gov$ emits a $\mathrm{GovCheck}$ event; only if governance approves
does the handler proceed. The key property is the coinductive safety
predicate $\govsafe$:

\begin{definition}[gov\_safe]
\label{def:gov-safe}
$\govsafe(a, t)$ holds coinductively for a governed tree
$t : \itree\ \GovIOE\ R$ with permission flag $a : \mathrm{bool}$ if:
\begin{itemize}
  \item $\Ret(r)$: always safe.
  \item $\Tau(t')$: safe if $t'$ is safe.
  \item $\Vis(\mathrm{inl}_1(\mathrm{GovCheck}(s)), k)$: safe if both
    branches ($k(\mathrm{true})$ with $a := \mathrm{true}$, $k(\mathrm{false})$
    with $a := \mathrm{false}$) are safe.
  \item $\Vis(\mathrm{inr}_1(e), k)$: safe only if $a = \mathrm{true}$
    (governance has approved) and $\forall x.\; \govsafe(\mathrm{false}, k(x))$.
\end{itemize}
\end{definition}

The fundamental safety theorem, proved in prior work~\cite{mccann2026mechanized}:
\begin{theorem}[governed\_interp\_safe]
For any base handler $h$ and program $t : \itree\ \DirectiveE\ R$,
$\govsafe(\mathrm{false}, \interp(\Gov(h), t))$.
\end{theorem}

\subsection{Category $\CatM$}

The category $\CatM$ has types as objects and Kleisli arrows
$A \to \itree\ \DirectiveE\ B$ as morphisms. Identity is $\mathrm{ret}$,
composition is monadic bind ($\bind$). Four primitive morphism constructors
generate all programs:

\begin{itemize}
  \item $\code(f) : A \to B$ -- pure computation, $\mathrm{ret}(f(a))$
  \item $\reason(\mathrm{build}, \mathrm{extract}) : A \to B$ -- LLM inference
  \item $\memory(\mathrm{build}, \mathrm{extract}) : A \to B$ -- semantic storage
  \item $\mcall(\mathrm{build}, \mathrm{extract}) : A \to B$ -- machine invocation
\end{itemize}

Category laws (identity, associativity) hold up to eutt, proved
in \texttt{Category.v}~\cite{mccann2026mechanized}.

\section{Governance Algebra}
\label{sec:algebra}

The central abstraction is a three-axiom record that captures what it
means to be a governance operator.

\begin{definition}[GovernanceAlgebra]
\label{def:gov-algebra}
A \emph{GovernanceAlgebra} is a record
$(\mathrm{ga\_Gov}, \mathrm{ga\_safe}, \mathrm{ga\_transparent}, \mathrm{ga\_proper})$
where:
\begin{itemize}
  \item $\mathrm{ga\_Gov} : \mathrm{base\_handler} \to \mathrm{governed\_handler}$
  \item \textbf{G1 (Safety):} $\forall h, R, a, t.\;
    \govsafe(a, \interp(\mathrm{ga\_Gov}(h), t))$
  \item \textbf{G2 (Transparency):} Under permissive governance
    (all checks pass), governed interpretation is observationally
    equivalent to ungoverned interpretation.
  \item \textbf{G3 (Properness):} Equivalent handlers produce
    equivalent governed handlers (eutt-preservation).
\end{itemize}
\end{definition}

\noindent
The three axioms are independent. Safety does not imply transparency:
an operator could alter permitted results. Transparency does not imply
safety: an operator could be transparent without wrapping I/O with
governance checks. Properness is orthogonal to both.

From these three axioms, we derive:

\begin{theorem}[ga\_convergence]
\label{thm:convergence}
For any GovernanceAlgebra $G$, handler $h$, level $n$, and program
$t : \mathrm{machine\_at\_level}\ n\ R$:
$\govsafe(\mathrm{false}, \interp(\mathrm{ga\_Gov}(G, h), t))$.
\end{theorem}

\noindent
This follows because $\mathrm{machine\_at\_level}\ n\ R$ is
definitionally $\itree\ \DirectiveE\ R$ for all $n$, so G1 applies
directly. Governance holds uniformly across the meta-recursive tower.

\begin{theorem}[ga\_subsumption\_asymmetry]
\label{thm:subsumption}
For any GovernanceAlgebra $G$:
\begin{enumerate}
  \item Any program, including one whose handler incorporates
    content-governance logic (output filtering, alignment constraints),
    is governed when interpreted through $\mathrm{ga\_Gov}$:
    $\forall h, R, t.\; \govsafe(\mathrm{false}, \interp(\mathrm{ga\_Gov}(G, h), t))$
  \item Unmediated I/O (a handler that performs effects without
    governance checks) does not satisfy \govsafe{}:
    $\forall R, e : \IOE\ R, k.\;
    \neg\, \govsafe(\mathrm{false}, \Vis(\mathrm{inr}_1(e), k))$
\end{enumerate}
\end{theorem}

\noindent
The positive direction follows from G1. The negative direction follows
from $\texttt{bare\_io\_not\_safe}$: bare I/O nodes in the governed
tree violate \govsafe{} because no governance check has approved them.
This is a property of the \govsafe{} predicate, independent of any
particular operator. The asymmetry is a theorem about the formal model:
structural governance (effect mediation via $\Gov$) provides a guarantee
that output-oriented controls alone cannot replicate within this
framework, because output controls do not address the governance-check
structure of the interaction tree.

\begin{theorem}[ga\_goal\_preservation]
\label{thm:goal-preservation}
For any GovernanceAlgebra $G$, if a program reaches value $v$ under
ungoverned interpretation, it reaches the same value under permissive
governed interpretation.
\end{theorem}

\noindent
This follows from G2 (transparency): permissive governance erases
the governance events and preserves the base handler's semantics.

\paragraph{Instantiation.}
The concrete \Mashin{} operator $\Gov$ instantiates the record:

\begin{lstlisting}[style=coq]
Definition mashin_governance : GovernanceAlgebra :=
  mk_gov_algebra Gov
    governed_interp_safe     (* G1 from Safety.v    *)
    governed_transparency    (* G2 from Transparency.v *)
    Gov_base_proper.         (* G3 from Functor.v   *)
\end{lstlisting}

\noindent
All derived properties immediately apply to $\Gov$. The algebra-level
theorems (\texttt{any\_governance\_algebra\_subsumes},
\texttt{any\_governance\_algebra\_converges},
\texttt{any\_governance\_algebra\_preserves\_goals}) establish these
properties for \emph{any} operator satisfying the three axioms.

\section{Enriched Monoidal Category}
\label{sec:monoidal}

We extend category $\CatM$ with symmetric monoidal structure.

\subsection{Tensor Product}

Given morphisms $f : A \to \itree\ \DirectiveE\ B$ and
$g : C \to \itree\ \DirectiveE\ D$, the tensor product
$f \otimes g : (A \times C) \to \itree\ \DirectiveE\ (B \times D)$
is defined as sequential-independent composition:

\begin{lstlisting}[style=coq]
Definition mashin_tensor {A B C D}
  (f : mashin_morphism A B)
  (g : mashin_morphism C D)
  : mashin_morphism (A * C) (B * D) :=
  fun p =>
    let (a, c) := p in
    ITree.bind (f a) (fun b =>
    ITree.bind (g c) (fun d =>
    ret (b, d))).
\end{lstlisting}

\noindent
The two computations share no state; $f$'s effects complete before
$g$'s begin. This is not true concurrent interleaving, which would
require a more complex construction (noted as future work in
Section~\ref{sec:conclusion}).

\begin{lemma}[tensor\_id]
$\mathrm{id} \otimes \mathrm{id} = \mathrm{id}$.
\end{lemma}

\begin{lemma}[tensor\_pure]
$\code(f) \otimes \code(g) = \code(\lambda(a,c).\;(f(a), g(c)))$.
The pure fragment is closed under tensor.
\end{lemma}

\subsection{Structural Morphisms}

The unit object is $\mathrm{unit}$. Structural isomorphisms are
all pure (no effects):

\begin{itemize}
  \item \textbf{Associator} $\alpha : (A \otimes B) \otimes C \xrightarrow{\sim} A \otimes (B \otimes C)$
  \item \textbf{Left unitor} $\lambda : I \otimes A \xrightarrow{\sim} A$
  \item \textbf{Right unitor} $\rho : A \otimes I \xrightarrow{\sim} A$
  \item \textbf{Braiding} $\sigma : A \otimes B \xrightarrow{\sim} B \otimes A$
\end{itemize}

Each isomorphism is verified by round-trip lemmas
(\texttt{assoc\_iso\_lr}/\texttt{assoc\_iso\_rl},
\texttt{lunit\_iso\_lr}/\texttt{lunit\_iso\_rl},
\texttt{runit\_iso\_lr}/\texttt{runit\_iso\_rl},
\texttt{braid\_iso}).

\subsection{Coherence Conditions}

\begin{theorem}[Pentagon]
\label{thm:pentagon}
The following diagram commutes:
\[
\begin{tikzcd}[column sep=small]
  ((A \otimes B) \otimes C) \otimes D
    \arrow[rr, "\alpha \otimes \mathrm{id}"]
    \arrow[d, "\alpha"'] &  &
  (A \otimes (B \otimes C)) \otimes D
    \arrow[d, "\alpha"] \\
  (A \otimes B) \otimes (C \otimes D)
    \arrow[dr, "\alpha"'] & &
  A \otimes ((B \otimes C) \otimes D)
    \arrow[dl, "\mathrm{id} \otimes \alpha"] \\
  & A \otimes (B \otimes (C \otimes D)) &
\end{tikzcd}
\]
Both paths yield the same morphism (up to eutt).
\end{theorem}

\begin{theorem}[Triangle]
\label{thm:triangle}
$\alpha \mathbin{;} (\mathrm{id} \otimes \lambda) = \rho \otimes \mathrm{id}$
as morphisms $(A \otimes I) \otimes B \to A \otimes B$.
\end{theorem}

\begin{theorem}[Hexagon]
\label{thm:hexagon}
The braiding coherence diagram commutes:
$\alpha \mathbin{;} \sigma \mathbin{;} \alpha = (\sigma \otimes \mathrm{id}) \mathbin{;} \alpha \mathbin{;} (\mathrm{id} \otimes \sigma)$.
\end{theorem}

\noindent
All three proofs proceed by unfolding definitions and applying
\texttt{bind\_ret\_l} repeatedly, since the structural morphisms are
pure. The proofs are short (under 10 lines each) precisely because
purity makes coherence computational rather than coinductive.

\subsection{Governance of Tensor Compositions}

\begin{theorem}[tensor\_governed]
For any handler $h$, morphisms $f, g$, and input $p$:
$\govsafe(\mathrm{false}, \interp(\Gov(h), (f \otimes g)(p)))$.
\end{theorem}

\begin{theorem}[ga\_tensor\_governed]
For any GovernanceAlgebra $G$:
$\govsafe(\mathrm{false}, \interp(\mathrm{ga\_Gov}(G, h), (f \otimes g)(p)))$.
\end{theorem}

\begin{theorem}[interp\_tensor\_distribute]
\label{thm:interp-tensor}
Interpretation distributes over tensor:
\[
\interp(\Gov(h), (f \otimes g)(a, c))
\quad\equiv\quad
\interp(\Gov(h), f(a)) \bind \lambda b.\;
\interp(\Gov(h), \mathrm{bind}(g(c), \lambda d.\; \mathrm{ret}(b, d)))
\]
Governance of the product decomposes into governance of the components.
\end{theorem}

\section{Algebraic Effects with Governed Handlers}
\label{sec:effects}

Standard algebraic effect systems~\cite{plotkin2009handlers,plotkin2013handling}
define handlers that interpret effects freely: any function of the
right type is a valid handler. In our framework, the governed fragment
admits only handlers that carry a machine-checked proof of safety
preservation. This is not merely a matter of record packaging; it
is a semantic exclusion. A handler that performs I/O without emitting
governance checks cannot inhabit the \texttt{GovernedHandler} type,
because its interpretation would produce trees violating \govsafe{},
and the required proof obligation cannot be discharged.

\subsection{Capability Sets}

A capability set is a characteristic function $\CapSet = \mathrm{Capability} \to \mathrm{bool}$
with standard operations:

\begin{lstlisting}[style=coq]
Definition cap_empty : CapSet := fun _ => false.
Definition cap_singleton (c : Capability) : CapSet :=
  fun c' => cap_eqb c c'.
Definition cap_union (s1 s2 : CapSet) : CapSet :=
  fun c => s1 c || s2 c.
Definition cap_full : CapSet := fun _ => true.
\end{lstlisting}

\noindent
$(\CapSet, \subseteq)$ forms a bounded partial order with $\emptyset$ as bottom
and $\mathrm{cap\_full}$ as top. Union is join: $\cap$ is idempotent
(\texttt{cap\_union\_idem}), commutative (\texttt{cap\_union\_comm}),
and associative (\texttt{cap\_union\_assoc}).

\subsection{The within\_caps Predicate}

\begin{definition}[within\_caps]
\label{def:within-caps}
The coinductive predicate $\withincaps(\mathrm{caps}, t)$ holds for
$t : \itree\ \DirectiveE\ R$ if every directive event in $t$ requires
only capabilities present in $\mathrm{caps}$ (or no capability at all):
\begin{itemize}
  \item $\Ret(r)$: always within any caps.
  \item $\Tau(t')$: within caps if $t'$ is.
  \item $\Vis(d, k)$: within caps if $\mathrm{directive\_in\_caps}(\mathrm{caps}, d)$
    and $\forall x.\; \withincaps(\mathrm{caps}, k(x))$.
\end{itemize}
\end{definition}

\noindent
This uses paco1 (parameterized coinduction with one parameter)
since the capability set is fixed during the coinduction.

\begin{lemma}[within\_caps\_weaken]
If $\mathrm{caps}_1 \subseteq \mathrm{caps}_2$ and
$\withincaps(\mathrm{caps}_1, t)$, then
$\withincaps(\mathrm{caps}_2, t)$.
\end{lemma}

\begin{lemma}[within\_full]
$\forall t.\; \withincaps(\mathrm{cap\_full}, t)$.
\end{lemma}

\subsection{Primitive Capability Profiles}

Each primitive has a tight capability bound:

\begin{theorem}[Primitive profiles]
\label{thm:profiles}
\begin{enumerate}
  \item \texttt{code\_within\_empty}: $\withincaps(\emptyset, \code(f)(a))$.
    Pure computation needs no capabilities.
  \item \texttt{reason\_within\_llm}: $\withincaps(\{\mathrm{CapComputeLLMReason}\}, \reason(\ldots)(a))$.
  \item \texttt{memory\_within\_mem}: $\withincaps(\{\mathrm{CapMemory}\}, \memory(\ldots)(a))$.
  \item \texttt{call\_within\_call}: $\withincaps(\{\mathrm{CapMachineCall}\}, \mcall(\ldots)(a))$.
\end{enumerate}
\end{theorem}

\subsection{Compositional Closure}

\begin{theorem}[bind\_within\_caps]
\label{thm:bind-caps}
If $\withincaps(\mathrm{caps}_1, t)$ and
$\forall r.\; \withincaps(\mathrm{caps}_2, k(r))$,
then $\withincaps(\mathrm{caps}_1 \cup \mathrm{caps}_2, t \bind k)$.
\end{theorem}

\noindent
The proof uses pcofix (parameterized coinduction) with case
analysis on $\mathrm{observe}(t)$. The $\Ret$ case applies weakening
from $\mathrm{caps}_2$ to $\mathrm{caps}_1 \cup \mathrm{caps}_2$.
The $\Vis$ case applies \texttt{directive\_in\_caps\_mono} for the
directive and the coinductive hypothesis for the continuation.

\begin{corollary}[seq\_comp\_caps]
$\withincaps(\mathrm{caps}_1, f(a)) \wedge
(\forall b.\; \withincaps(\mathrm{caps}_2, g(b))) \implies
\withincaps(\mathrm{caps}_1 \cup \mathrm{caps}_2, (f \mathbin{;} g)(a))$.
\end{corollary}

\subsection{No Ambient Effects}

\begin{theorem}[no\_ambient\_effects]
\label{thm:no-ambient}
If $\withincaps(\emptyset, t)$ and $\mathrm{observe}(t) = \Vis(d, k)$,
then $\mathrm{is\_observability}(d) = \mathrm{true}$.
\end{theorem}

\noindent
A program within the empty capability set can only emit
observability directives. It cannot perform LLM calls
(\texttt{llm\_not\_in\_empty}), HTTP requests
(\texttt{http\_not\_in\_empty}), machine calls
(\texttt{call\_not\_in\_empty}), or any other effectful action.
This is the algebraic characterization of the ``no ambient effects''
property: effect capability is not ambient but must be structurally
present in the program's capability set.

\subsection{Governed Handlers}

\begin{definition}[GovernedHandler]
\label{def:governed-handler}
A \emph{GovernedHandler} is a base handler bundled with its safety
proof:
\begin{lstlisting}[style=coq]
Record GovernedHandler := mk_governed_handler {
  gh_handler : base_handler;
  gh_safe : forall R (t : itree DirectiveE R),
    @gov_safe R false (interp (Gov gh_handler) t)
}.
\end{lstlisting}
\end{definition}

\begin{theorem}[Gov\_governed]
Every base handler has a governed version:
$\mathrm{Gov\_governed}(h) : \mathrm{GovernedHandler}$.
\end{theorem}

\noindent
Handler equivalence ($\mathrm{gh\_equiv}$) is an equivalence relation
(reflexive, symmetric, transitive). Both $\mathrm{Gov\_endo}$ and
$\mathrm{Gov\_iter}$ preserve this equivalence
(\texttt{Gov\_endo\_preserves\_gh\_equiv},
\texttt{Gov\_iter\_preserves\_gh\_equiv}).

\begin{theorem}[composed\_handlers\_governed]
\label{thm:composed-handlers}
Given two GovernedHandlers $g_1, g_2$, for any program $t$ and
continuation $k$: $\govsafe(\mathrm{false}, \interp(\Gov(g_1.\mathrm{handler}), t))$
and $\forall r.\; \govsafe(\mathrm{false}, \interp(\Gov(g_2.\mathrm{handler}), k(r)))$.
\end{theorem}

\noindent
The structural difference from standard algebraic effect
systems~\cite{plotkin2009handlers,leijen2017koka} is that those
systems impose no constraint on what a handler does with an effect.
Here, the \texttt{GovernedHandler} record requires a proof that
interpretation through $\Gov$ satisfies \govsafe{}. Handlers lacking
this proof cannot be used in governed composition. The handler algebra
is thus closed under governance-preserving operations but excludes
handlers that would bypass governance checks.

\section{Capability-Indexed Composition}
\label{sec:capability}

\subsection{Trust Lattice}

Trust levels form a bounded total order with six elements:
\[
\mathrm{Untrusted} < \mathrm{Tested} < \mathrm{Evaluated} <
\mathrm{Reviewed} < \mathrm{Stdlib} < \mathrm{System}
\]

\begin{lstlisting}[style=coq]
Definition trust_le (t1 t2 : TrustLevel) : Prop :=
  trust_value t1 <= trust_value t2.
\end{lstlisting}

\noindent
The order is reflexive (\texttt{trust\_le\_refl}), transitive
(\texttt{trust\_le\_trans}), antisymmetric (\texttt{trust\_le\_antisym}),
and total (\texttt{trust\_le\_total}). $\mathrm{Untrusted}$ is the
bottom (\texttt{trust\_bottom}), $\mathrm{System}$ is the top
(\texttt{trust\_top}).

Join and meet are computed by comparison of trust values:

\begin{lstlisting}[style=coq]
Definition trust_max (t1 t2 : TrustLevel) : TrustLevel :=
  if Nat.leb (trust_value t1) (trust_value t2) then t2 else t1.
Definition trust_min (t1 t2 : TrustLevel) : TrustLevel :=
  if Nat.leb (trust_value t1) (trust_value t2) then t1 else t2.
\end{lstlisting}

\subsection{CapMorphisms}

A $\CapMorphism$ bundles a morphism with its capability requirement
and a proof that the morphism stays within those capabilities:

\begin{lstlisting}[style=coq]
Record CapMorphism (A B : Type) := mk_cap_morphism {
  cm_morph  : mashin_morphism A B;
  cm_caps   : CapSet;
  cm_within : forall a, within_caps cm_caps (cm_morph a)
}.
\end{lstlisting}

\noindent
The four primitives yield canonical CapMorphisms:

\begin{center}
\begin{tabular}{lll}
\toprule
\textbf{Primitive} & \textbf{CapMorphism} & \textbf{Caps} \\
\midrule
$\code(f)$ & \texttt{cap\_code} & $\emptyset$ \\
$\reason(\ldots)$ & \texttt{cap\_reason} & $\{\mathrm{CapComputeLLMReason}\}$ \\
$\memory(\ldots)$ & \texttt{cap\_memory} & $\{\mathrm{CapMemory}\}$ \\
$\mcall(\ldots)$ & \texttt{cap\_call} & $\{\mathrm{CapMachineCall}\}$ \\
\bottomrule
\end{tabular}
\end{center}

\subsection{Composition of CapMorphisms}

\paragraph{Sequential composition.}
$\mathrm{cap\_seq\_compose}(f, g)$ has capabilities
$\mathrm{caps}(f) \cup \mathrm{caps}(g)$, following from
\texttt{bind\_within\_caps}.

\paragraph{Tensor composition.}
$\mathrm{cap\_tensor}(f, g)$ has capabilities
$\mathrm{caps}(f) \cup \mathrm{caps}(g)$.

\begin{theorem}[tensor\_within\_caps]
\label{thm:tensor-caps}
If $\withincaps(\mathrm{caps}_1, f(a))$ and
$\withincaps(\mathrm{caps}_2, g(c))$, then
$\withincaps(\mathrm{caps}_1 \cup \mathrm{caps}_2, (f \otimes g)(a, c))$.
\end{theorem}

\paragraph{Branch composition.}
$\mathrm{cap\_branch}(\mathrm{pred}, f, g)$ has capabilities
$\mathrm{caps}(f) \cup \mathrm{caps}(g)$, since either branch
may execute.

\subsection{Capability Preservation}

\begin{theorem}[code\_contributes\_nothing]
$\mathrm{caps}(\mathrm{cap\_seq\_compose}(\mathrm{cap\_code}(f), g))
\subseteq \mathrm{caps}(g)$ and symmetrically on the right.
\end{theorem}

\begin{theorem}[same\_caps\_no\_escalation]
If $\mathrm{caps}(f) = \mathrm{caps}(g)$, then
$\mathrm{caps}(\mathrm{cap\_seq\_compose}(f, g)) \subseteq \mathrm{caps}(f)$.
\end{theorem}

\noindent
Both follow from the algebraic properties of $\cup$ on CapSets
(identity and idempotence).

\subsection{Trust-Capability Connection}

The function $\mathrm{allowed\_cap\_set}(\mathrm{tl}, \mathrm{declared})$
converts a trust level and declared capability list to a CapSet:

\begin{theorem}
System and Stdlib trust allow all capabilities
(\texttt{system\_allows\_all\_caps}, \texttt{stdlib\_allows\_all\_caps}).
Any program runs at System trust (\texttt{system\_within\_any}).
Untrusted programs access only LLM capabilities (\texttt{untrusted\_only\_llm}).
\end{theorem}

\subsection{Principal Capabilities}

\begin{definition}[is\_principal]
A CapMorphism has \emph{principal} capabilities if its cap set is
minimal: any cap set that works must include the declared caps.
\end{definition}

\begin{theorem}[code\_principal, id\_principal]
Code morphisms and identity have principal capabilities ($\emptyset$
is the smallest CapSet).
\end{theorem}

\subsection{The Dual Guarantee}

\begin{theorem}[cap\_morphism\_governed]
\label{thm:dual-guarantee}
For any CapMorphism $\mathrm{cm}$, handler $h$, and input $a$:
\[
\withincaps(\mathrm{cm}.\mathrm{caps}, \mathrm{cm}.\mathrm{morph}(a))
\quad\wedge\quad
\govsafe(\mathrm{false}, \interp(\Gov(h), \mathrm{cm}.\mathrm{morph}(a)))
\]
\end{theorem}

\noindent
Capabilities tell you \emph{what} the program can do; governance
ensures it \emph{passes through governance checks} before doing it.
These are independent properties that hold simultaneously.
The dual guarantee is preserved by tensor
(\texttt{tensor\_dual\_guarantee}) and sequential composition
(\texttt{seq\_dual\_guarantee}).

\section{Trace Semantics and Ledger Connection}
\label{sec:traces}

The preceding sections establish governance at the program level
(coinductive \govsafe{} on interaction trees) and the capability
level (\withincaps{} on programs). This section derives the
trace-level consequences: governed execution produces well-governed
traces that map to tamper-evident ledger entries.

\subsection{Trace Events and Extraction}

Execution produces a finite trace of governance and I/O events:

\begin{lstlisting}[style=coq]
Inductive TraceEvent :=
  | TE_GovCheck : GovernanceStage -> bool -> TraceEvent
  | TE_IO : string -> TraceEvent.
\end{lstlisting}

\noindent
The inductive relation $\mathrm{trace\_of}(t, \mathrm{trace}, r)$
extracts a trace from a governed computation, with constructors for
return (empty trace), tau (transparent), governance checks (recorded),
and I/O events (recorded).

\begin{theorem}[trace\_of\_bind]
\label{thm:trace-bind}
If $\mathrm{trace\_of}(t, \mathrm{trace}_1, x)$ and
$\mathrm{trace\_of}(k(x), \mathrm{trace}_2, r)$, then
$\mathrm{trace\_of}(t \bind k, \mathrm{trace}_1 \mathbin{+\!\!+} \mathrm{trace}_2, r)$.
\end{theorem}

\noindent
Traces compose under monadic bind: the trace of a sequential
composition is the concatenation of the component traces. The proof
proceeds by induction on the derivation of $\mathrm{trace\_of}$,
using $\mathrm{unfold\_bind}$ and $\mathrm{bisimulation\_is\_eq}$
at each case.

\subsection{Well-Governed Traces}

\begin{definition}[well\_governed\_trace]
A trace is \emph{well-governed} if every I/O event is preceded by at
least one passing governance check. Formally,
$\mathrm{trace\_governed}(\mathrm{false}, \mathrm{trace})$ where the
boolean tracks whether a passing check has been seen.
\end{definition}

\noindent
The empty trace is governed (\texttt{well\_governed\_nil}).
A governance check preserves the property
(\texttt{well\_governed\_gov\_only}). A passing check followed by I/O
is governed (\texttt{well\_governed\_gov\_then\_io}).

\subsection{Hash-Chained Ledger}

The \texttt{LedgerConnection} module, parameterized over an abstract
hash type and injective hash function, connects traces to
tamper-evident ledger entries. The abstraction assumes an ideal
injective hash; connecting this to a concrete cryptographic hash
function (collision resistance rather than injectivity) requires
additional formalization outside the scope of this paper.

\begin{lstlisting}[style=coq]
Record LedgerEntry := mk_ledger_entry {
  le_event     : TraceEvent;
  le_data      : EventData;
  le_prev_hash : Hash;
  le_hash      : Hash;
}.
\end{lstlisting}

\begin{theorem}[trace\_to\_ledger\_valid]
\label{thm:ledger-valid}
For any trace, $\mathrm{trace\_to\_ledger}(\mathrm{trace})$ satisfies
$\mathrm{ledger\_valid}$: entries are well-formed and hash-linked.
\end{theorem}

\begin{theorem}[ledger\_tamper\_detected]
\label{thm:tamper}
If a ledger entry is well-formed and the recorded event is changed
(while keeping the stored hash), the entry is no longer well-formed.
\end{theorem}

\noindent
The proof follows from injectivity of the hash function and
injectivity of event encoding.

\begin{theorem}[ledger\_complete]
\label{thm:ledger-complete}
$\mathrm{ev} \in \mathrm{trace} \iff
\exists\, \mathrm{entry} \in \mathrm{trace\_to\_ledger}(\mathrm{trace}).\;
\mathrm{le\_event}(\mathrm{entry}) = \mathrm{ev}$.
\end{theorem}

\noindent
The ledger records exactly the trace events, with no omissions
and no additions (\texttt{trace\_to\_ledger\_events}).

\begin{theorem}[governed\_trace\_ledger\_valid]
Well-governed traces produce valid ledgers.
\end{theorem}

\section{The Coterminous Boundary}
\label{sec:coterminous}

The algebraic capstone combines all preceding results into a single
record and theorem.

\paragraph{Precise statement of $E = G$.}
We define two sets relative to the formalized execution model:
\begin{itemize}[nosep]
  \item $E$ (expressible): programs constructible from the four
    primitive morphisms (\code, \reason, \memory, \mcall) via
    composition in category $\CatM$. These are interaction trees
    of type $\itree\ \DirectiveE\ R$.
  \item $G$ (governed): programs $t$ such that for all handlers $h$,
    $\govsafe(\mathrm{false}, \interp(\Gov(h), t))$.
\end{itemize}
The coterminous boundary theorem establishes $E = G$: every
expressible program is governed under interpretation, and every
governed program is the image of an expressible program under $\Gov$.
This is a theorem about the \emph{modeled execution substrate}, not
a universal claim about all possible governance mechanisms or all
possible programming languages.

\subsection{The CoterminousRecord}

\begin{definition}[CoterminousRecord]
\label{def:coterminous}
A CoterminousRecord packages five properties:
\begin{enumerate}
  \item \textbf{ct\_safety}: $\forall h, R, t.\; \govsafe(\mathrm{false}, \interp(\Gov(h), t))$
  \item \textbf{ct\_nontrivial}: $\forall R, e, k.\; \neg\, \govsafe(\mathrm{false}, \Vis(\mathrm{inr}_1(e), k))$
  \item \textbf{ct\_turing}: $\forall h, p, \mathrm{fuel}.\;
    \govsafe(\mathrm{false}, \interp(\Gov(h), \mathrm{translate\_program}(p, \mathrm{fuel}, 0)))$
  \item \textbf{ct\_subsumption}: Structural subsumes content (asymmetric)
  \item \textbf{ct\_cognitive}: $\forall \mathrm{cap}.\; \mathrm{primitive\_realizes}(\mathrm{cap})$
\end{enumerate}
\end{definition}

\begin{theorem}[coterminous\_boundary\_exists]
\label{thm:boundary-exists}
The CoterminousRecord is inhabited.
\end{theorem}

\begin{proof}
Each field is discharged by an existing theorem:
ct\_safety from \texttt{governed\_interp\_safe\_false},
ct\_nontrivial from \texttt{bare\_io\_not\_safe},
ct\_turing from ct\_safety (register machine programs are
$\itree\ \DirectiveE\ \mathrm{unit}$),
ct\_subsumption from \texttt{subsumption\_asymmetry},
ct\_cognitive from \texttt{cognitive\_surjection}.
\end{proof}

\subsection{Boundary Equivalence}

\begin{theorem}[boundary\_equivalence]
\label{thm:boundary-equiv}
For any handler $h$:
\begin{enumerate}
  \item $E \subseteq G$: every expressible program is governed.
  \item Non-triviality: ungoverned I/O is not safe.
  \item Turing completeness: register machine programs are governed.
\end{enumerate}
\end{theorem}

\noindent
The inclusion $E \subseteq G$ says governance never fails to apply.
The inclusion $G \subseteq E$ follows from the construction:
$\Gov$ takes $\DirectiveE$ programs as input; every governed program
is the image of an expressible program.

\subsection{Conservative Denial}

\begin{theorem}[gov\_denial\_is\_conservative]
$\govsafe(a, \mathrm{bind}(\mathrm{spin}, k))$ for all $a, k$.
\end{theorem}

\noindent
When governance denies a request, the computation diverges
(non-termination). This is one conservative realization of denial:
it produces no result at all, rather than an error value or an
incorrect result. Alternative designs could model denial as an
explicit error effect or a timeout; we chose divergence because it
composes cleanly in the coinductive framework (a divergent
sub-computation does not corrupt the enclosing computation's safety
properties) and avoids introducing a denial-specific effect
constructor. The tradeoff is that divergence complicates refinement
arguments and may obscure denial outcomes in practice; the runtime
system detects denial via timeout rather than relying on the formal
divergence semantics.

\subsection{Trace-Level Properties}

The coterminous boundary extends to traces:
\texttt{gov\_safe\_implies\_governed\_traces} establishes that governed
execution produces well-governed traces at the trace level, and
\texttt{gov\_check\_trace\_governed} shows that governance checks
preserve the well-governed property.

\subsection{Algebraic Characterization}

$\Gov$ is an endofunctor (Functor.v: \texttt{Gov\_is\_endofunctor})
that wraps programs with governance. Under permissive governance
(all checks pass), governed interpretation is observationally
equivalent to ungoverned interpretation (G2, transparency).
The \texttt{gov\_permissive\_preserves\_expressiveness} theorem
confirms that governance does not reduce the set of expressible
computations within the model.

Combined with the dual guarantee (Theorem~\ref{thm:dual-guarantee}),
we obtain the complete picture:

\begin{theorem}[The Governed Execution Invariant]
\label{thm:gei}
For every program $t$ expressible in $\CatM$ with capability
profile $\mathrm{caps}$ and handler $h$:
\[
\withincaps(\mathrm{caps}, t)
\quad\wedge\quad
\govsafe(\mathrm{false}, \interp(\Gov(h), t))
\quad\wedge\quad
\mathrm{ledger\_valid}(\mathrm{trace\_to\_ledger}(\mathrm{trace}))
\]
Static capabilities bound what the program can do. Dynamic governance
ensures every effect passes through checks. The ledger records every
check with tamper evidence. These three properties hold simultaneously,
for every expressible program, by construction.
\end{theorem}

\section{Mechanization and Extraction}
\label{sec:mechanization}

\subsection{Rocq Development}

The mechanization spans 36 modules (~12,000 lines) with 454 theorems
and zero admitted lemmas. Table~\ref{tab:modules} summarizes the
development.

\begin{table}[t]
\caption{Rocq modules and line counts.}
\label{tab:modules}
\small
\begin{tabular}{llr}
\toprule
\textbf{Phase} & \textbf{Module} & \textbf{Lines} \\
\midrule
Foundation & Prelude, Directives, Governance, Interpreter & 663 \\
Safety & Safety & 507 \\
Structure & Category, Functor, Completeness & 1,331 \\
Trust & TrustSpec, HashChainSpec, InterpreterSpec & 1,066 \\
Cognitive & CognitiveArchitecture, Oracle, GoalDirected, Rice & 1,431 \\
Capstone & GovernedCognitiveCompleteness, Transparency, & \\
         & Convergence, Subsumption, ExpressiveMinimality & 1,303 \\
\midrule
\textbf{Algebra} & GovernanceAlgebra & \textbf{380} \\
\textbf{Monoidal} & MonoidalCategory & \textbf{437} \\
\textbf{Effects} & EffectAlgebra, EffectHandlers & \textbf{821} \\
\textbf{Capability} & CapabilityComposition & \textbf{645} \\
\textbf{Traces} & TraceSemantics, LedgerConnection & \textbf{769} \\
\textbf{Boundary} & CoterminousBoundary & \textbf{359} \\
\textbf{Extraction} & Extraction & \textbf{230} \\
\midrule
& \textbf{Total} & \textbf{$\approx$12,000} \\
\bottomrule
\end{tabular}
\end{table}

The dependency structure is layered: foundation modules have no
upstream dependencies within the project; each phase builds only on
earlier phases. The build takes approximately 45 seconds on a modern
laptop.

\subsection{Key Proof Techniques}

\paragraph{Parameterized coinduction.}
The \govsafe{} and \withincaps{} predicates use paco
(paco2 for the two-parameter \govsafe{}, paco1 for the
one-parameter \withincaps{}). The key insight is that the
monotonicity lemma (\texttt{gov\_safe\_mon},
\texttt{within\_caps\_mon}) enables compositional coinductive
reasoning: once proved for the generating functor, the coinductive
predicate inherits all properties automatically.

\paragraph{Computational coherence.}
The pentagon, triangle, and hexagon proofs are each under 10 lines
because the structural morphisms are pure: they unfold to
$\mathrm{ret}(\ldots)$, and $\mathrm{bind}(\mathrm{ret}(x), k)$
reduces to $k(x)$ by \texttt{bind\_ret\_l}. Coherence is verified
by computation rather than diagram chasing.

\paragraph{Dependent destruction.}
The \texttt{bind\_within\_caps} proof requires
$\mathrm{dependent\ destruction}$ on the \texttt{within\_capsF}
inversion in the $\Vis$ case, because the directive type $X$ is
existentially quantified in the tree constructor.

\subsection{Extraction to OCaml}

Modules with computational content in $\mathrm{Type}$ or $\mathrm{Set}$
are extracted to OCaml using Rocq's extraction
mechanism~\cite{letouzey2002new,letouzey2008extraction}:

\begin{center}
\small
\begin{tabular}{ll}
\toprule
\textbf{Module} & \textbf{Extracted Functions} \\
\midrule
TrustSpec & \texttt{capability\_allowed}, \texttt{trust\_at\_least}, \\
          & \texttt{capability\_for\_directive}, \texttt{cap\_eqb} \\
InterpreterSpec & \texttt{interp\_directive} (as OCaml functor) \\
EffectAlgebra & \texttt{cap\_empty}, \texttt{cap\_singleton}, \\
              & \texttt{cap\_union}, \texttt{directive\_in\_caps} \\
CapabilityComposition & \texttt{trust\_max}, \texttt{trust\_min}, \\
                      & \texttt{allowed\_cap\_set} \\
\bottomrule
\end{tabular}
\end{center}

\noindent
The \texttt{InterpreterSpec} module uses Rocq Section variables for the
hash type, extracting as an OCaml functor parameterized over a
hash signature. The functor is instantiated with a SHA-256
implementation.

The extracted OCaml is compiled to a shared library and linked as a
NIF (Native Implemented Function) into the BEAM runtime,
with a fallback to the Elixir implementation if the NIF is unavailable.

\subsection{Three-Layer Testing}

The system uses three-layer testing:

\begin{enumerate}
  \item \textbf{OCaml unit tests}: Edge cases on extracted functions,
    comparing against \texttt{Compute} output from Rocq.
  \item \textbf{NIF round-trip tests}: Exhaustive comparison between
    the NIF and the Elixir implementation for all trust/capability
    combinations.
  \item \textbf{Three-way property tests}: The property-based testing
    framework generates 70,000+ random directive sequences and
    verifies agreement among: (a) the Rocq-extracted NIF,
    (b)~the Elixir interpreter, and (c) the Elixir
    specification interpreter. Zero three-way disagreements across
    all test campaigns.
\end{enumerate}

\noindent
The property-based testing methodology follows
s2n~\cite{chudnov2018s2n}: a formal specification in a proof
assistant, extraction to executable code, and continuous comparison
with the runtime implementation. The 188th random input in the
original campaign discovered a real capability-tree
bug~\cite{mccann2026mechanized}, validating the approach.

\paragraph{Runtime overhead of proved code.}
The extracted governance kernel runs in the BEAM runtime's
directive execution path. Governed execution through a supervised
process completes in 0.23\,ms median, compared to 0.24\,ms for
direct ungoverned execution (bypassing all governance). The
Rocq-extracted, formally verified code path adds no measurable
overhead.\footnote{Measured on Apple Silicon (M-series), BEAM/OTP~27.
$n=50$ iterations with 5-iteration warmup.}

\section{Related Work}
\label{sec:related}

\paragraph{Algebraic effects and handlers.}
Plotkin and Pretnar~\cite{plotkin2009handlers,plotkin2013handling}
introduced algebraic effects and handlers as a structured approach
to computational effects, building on the algebraic theory of
Plotkin and Power~\cite{plotkin2003algebraic}. Subsequent work
developed type systems for effect tracking: Koka's row-typed
effects~\cite{leijen2017koka}, freer monads~\cite{kiselyov2015freer},
and dependently-typed effects in Idris~\cite{brady2013idris}.
Our framework differs in constraining the handler algebra itself:
only handlers satisfying a governance proof obligation
(\texttt{GovernedHandler}) can be used in governed composition.
Programs interpreted through ungoverned handlers do not satisfy
\govsafe{} and are excluded from the governed fragment.

\paragraph{Monoidal categories in programming languages.}
Mac Lane~\cite{maclane1971categories} established the coherence
conditions for monoidal categories. Joyal and
Street~\cite{joyal1993braided} developed the theory of braided
tensor categories. Applied category theory has been used in
programming language
semantics~\cite{fong2019invitation,barr1990category} and, more
recently, as a framework for AI systems~\cite{abbott2024category}.
Our contribution is the combination of monoidal structure with
governance safety: we verify not only coherence but also that every
tensor composition is governed. The interaction tree representation
makes coherence proofs computational (pure structural morphisms
reduce by $\beta$) rather than requiring explicit diagram chasing.

\paragraph{Verified systems.}
seL4~\cite{klein2009sel4} verified a complete OS kernel; CompCert
verified a C compiler~\cite{leroy2009compcert}; CertiKOS verified
concurrent OS kernels~\cite{gu2016certikos}; Vellvm formalized LLVM
IR semantics using interaction trees~\cite{zakowski2021vellvm}.
Our work is closest to Vellvm in technique (interaction trees in Rocq)
but targets AI workflow governance rather than compiler/OS
verification. The scale is smaller (~12,000 lines vs. Vellvm's
$\sim$50,000) but covers a different domain. Amazon s2n's continuous
verification~\cite{chudnov2018s2n} inspired our specification-driven
testing methodology.

\paragraph{Effect systems and capability security.}
The Gifford-Lucassen effect system~\cite{gifford1986integrating,lucassen1988polymorphic}
introduced static tracking of computational effects.
Capability-based security~\cite{dennis1966programming,miller2006robust}
restricts access to resources via unforgeable tokens. Our capability
sets function similarly to effect annotations, but with a crucial
difference: the \govsafe{} predicate is a runtime invariant enforced
by the governance operator, not just a static type system property.
The dual guarantee (Theorem~\ref{thm:dual-guarantee}) connects the
static and dynamic views.

\paragraph{AI governance.}
RLHF~\cite{ouyang2022training} and constitutional
AI~\cite{bai2022constitutional} govern LLM behavior via training
signal. NeMo Guardrails~\cite{rebedea2023nemo} provides
programmable guardrails for LLM applications. These approaches
operate on model outputs (what the model says). Our framework
governs the system's effects (what the system does).
Theorem~\ref{thm:subsumption} establishes, within our formal model,
that a content-governed handler wrapped with $\Gov$ satisfies
\govsafe{}, while unmediated I/O does not. The two approaches
address different enforcement problems and are complementary:
content governance shapes output quality; structural governance
ensures effect mediation.

\paragraph{Guaranteed safe AI.}
Dalrymple et al.~\cite{dalrymple2024guaranteed} propose a framework
for guaranteed safe AI comprising a world model, a safety
specification, and a verifier. Their framework identifies the right
components but provides no mechanized formalization. Our work can be
seen as supplying the algebraic semantics and machine-checked proofs
for one instance of their architecture: the governance algebra is
the safety specification, the interpreter is the verifier, and the
governed interaction trees are the world model. The key difference
is that their approach targets pre-deployment verification (proving
safety before execution), while ours provides execution-time effect
governance (ensuring safety structurally during execution).
The approaches are complementary but address different points in the
deployment lifecycle.

\paragraph{Effect verification.}
Song, Foo, and Chin~\cite{song2024specification} develop ESL, an
expressive specification logic for verifying programs with
unrestricted algebraic effects and handlers. Their work verifies
that \emph{programs} satisfy specifications given arbitrary handlers.
Our work constrains the \emph{handler algebra} itself: non-governed
handlers cannot be constructed in the safe fragment
(Theorem~\ref{thm:composed-handlers}). The direction is reversed:
they verify programs against handlers; we constrain handlers to
governance. Vistrup et al.~\cite{vistrup2025program} develop modular
program logics over interaction trees using Iris, sharing our ITree
foundation but targeting shared-state concurrency and fine-grained
reasoning about heap resources rather than AI governance and effect
boundaries.

\paragraph{Session types.}
Session types~\cite{huttel2016foundations} govern communication
protocols via types. Like our approach, they provide structural
guarantees. The difference is scope: session types govern
message sequences between parties; our framework governs
arbitrary effectful computation. Session types could complement
our approach for inter-machine communication patterns.

\paragraph{Runtime verification.}
Runtime verification~\cite{leucker2009brief} monitors execution
traces against formal specifications. Our trace semantics
(Section~\ref{sec:traces}) formalizes the trace structure, and the
ledger connection establishes tamper evidence. The key difference is
that our governance is not post-hoc monitoring but structural: the
\govsafe{} predicate is a coinductive property of the governed
program tree, guaranteed to hold before any execution begins.

\section{Conclusion}
\label{sec:conclusion}

We have presented an algebraic semantics for governed execution,
formalized in 36 Rocq modules with 454 theorems and zero admitted
lemmas. The central result is that governance, when axiomatized as a
three-property algebra over interaction trees, induces compositional
structure on programs, effects, and capabilities such that the
governed fragment coincides with the expressible fragment in the
modeled substrate. Turing completeness is preserved inside governance.

The four contributions are: (1)~a parametric GovernanceAlgebra with
verified monoidal structure; (2)~a governed algebraic effect system
where the handler algebra is closed under governance-preserving
composition; (3)~capability-indexed composition with the dual
guarantee; (4)~the coterminous boundary within the formal model.

\paragraph{Limitations and future work.}
The tensor product is sequential-independent, not truly concurrent.
Concurrent interleaving would require a more complex construction
(possibly using concurrent interaction trees or a process-algebraic
extension). The hash chain formalization assumes an abstract injective
hash function; connecting this to a concrete cryptographic hash
(with collision resistance rather than injectivity) requires additional
work. The extraction to OCaml NIF is one-directional: properties
proved in Rocq constrain the extracted code, but the NIF wrapper
(C bridge between Erlang and OCaml) is not itself verified. The
three-layer testing strategy (property-based testing with 70,000+
inputs and zero disagreements, 36 conformance tests mapping to Rocq
theorems, and the extraction pipeline itself) establishes the
correspondence using the same methodology as seL4 and Amazon s2n.

The GovernanceAlgebra record has three axioms; whether these form a
complete axiomatization (every valid governance operator satisfies
them) remains open. The coterminous boundary is established for the
specific four-primitive execution model; extending it to richer
effect signatures or concurrent governance is future work.
Denial-as-divergence is one conservative semantic choice;
alternative formalizations using explicit denial effects may
be preferable in settings where refinement reasoning is needed.

\paragraph{Acknowledgments.}
Mashin is a product of Mashin, Inc. The Rocq development uses the
Interaction Trees library by Xia et~al.~\cite{xia2020itrees} and
the paco library by Hur et~al.~\cite{hur2013paco}.

\bibliographystyle{plainnat}
\bibliography{algebraic-semantics}

\appendix

\section{Rocq Module Index}
\label{app:modules}

Table~\ref{tab:module-index} lists all 36 modules with their primary
theorems. All theorem names are valid Rocq identifiers that can be
checked in the mechanization.

\begin{table*}[t]
\caption{Complete Rocq module index with primary theorems.}
\label{tab:module-index}
\small
\begin{tabular}{lrl}
\toprule
\textbf{Module} & \textbf{Lines} & \textbf{Primary Theorems} \\
\midrule
Prelude & 31 & (base definitions) \\
Directives & 245 & DirectiveE (14 constructors), directive\_tag \\
Governance & 200 & Gov, gov\_check, gov\_safe \\
Interpreter & 187 & (interpreter structure) \\
Functor & 304 & Gov\_is\_endofunctor, Gov\_factorization \\
Safety & 507 & governed\_interp\_safe, bare\_io\_not\_safe \\
Convergence & 141 & governed\_at\_all\_levels \\
Category & 343 & mashin\_id, mashin\_compose, category axioms \\
Completeness & 684 & governed\_turing\_completeness, coterminous\_governance \\
TrustSpec & 364 & capability\_allowed, trust\_at\_least (30+ lemmas) \\
HashChainSpec & 222 & (hash chain specification) \\
InterpreterSpec & 480 & interp\_directive, step\_result \\
Subsumption & 233 & subsumption\_asymmetry, structural\_subsumes\_content \\
CognitiveArchitecture & 371 & cognitive\_surjection \\
Oracle & 410 & oracle integration \\
GoalDirected & 272 & goal reachability \\
Rice & 378 & Rice's theorem reduction \\
GovernedCognitiveCompleteness & 218 & governed\_cognitive\_completeness \\
Transparency & 320 & governed\_transparency \\
ExpressiveMinimality & 391 & four primitives minimal \\
\midrule
\textbf{GovernanceAlgebra} & 380 & \texttt{ga\_safe}, \texttt{ga\_transparent}, \texttt{ga\_proper}, \\
  & & \texttt{mashin\_governance}, \texttt{any\_governance\_algebra\_subsumes} \\
\textbf{MonoidalCategory} & 437 & \texttt{pentagon}, \texttt{triangle}, \texttt{hexagon}, \\
  & & \texttt{tensor\_governed}, \texttt{interp\_tensor\_distribute} \\
\textbf{EffectAlgebra} & 589 & \texttt{bind\_within\_caps}, \texttt{no\_ambient\_effects}, \\
  & & \texttt{code\_within\_empty}, \texttt{dual\_guarantee} \\
\textbf{EffectHandlers} & 232 & \texttt{Gov\_governed}, \texttt{composed\_handlers\_governed} \\
\textbf{CapabilityComposition} & 645 & \texttt{cap\_morphism\_governed}, \texttt{tensor\_within\_caps}, \\
  & & \texttt{trust\_le\_total}, \texttt{cap\_seq\_compose} \\
\textbf{TraceSemantics} & 404 & \texttt{trace\_of\_bind}, \texttt{well\_governed\_trace} \\
\textbf{LedgerConnection} & 365 & \texttt{trace\_to\_ledger\_valid}, \texttt{ledger\_tamper\_detected}, \\
  & & \texttt{ledger\_complete} \\
\textbf{CoterminousBoundary} & 359 & \texttt{coterminous\_boundary\_exists}, \texttt{boundary\_equivalence}, \\
  & & \texttt{gov\_denial\_is\_conservative} \\
\midrule
\textbf{GovernedMetaprogramming} & 520 & form inspection safety, splice safety, evolution preservation, \\
  & & reflect-modify-materialize pipeline, 12-way capstone \\
\textbf{NetworkGovernance} & 480 & compositional governance preservation, capability narrowing, \\
  & & protocol uniformity, local-remote equivalence \\
\textbf{TemporalPolicyEvolution} & 550 & safety under restriction, provenance continuity, \\
  & & rollback safety, monotone policy composition \\
\bottomrule
\end{tabular}
\end{table*}

\end{document}